\documentclass[sigconf]{aamas} 

\usepackage{balance} 
\usepackage{algorithm}
\usepackage{algorithmic}
\usepackage{amsmath}
\usepackage{amsthm}

\usepackage{amssymb}
\usepackage{subfig}
\newtheorem{theorem}{Theorem}

\newtheorem{property}{Property}
\newtheorem{lemma}{Lemma}
\usepackage{color} 

\setcopyright{ifaamas}
\acmConference[AAMAS '23]{Proc.\@ of the 22nd International Conference
on Autonomous Agents and Multiagent Systems (AAMAS 2023)}{May 29 -- June 2, 2023}
{London, United Kingdom}{A.~Ricci, W.~Yeoh, N.~Agmon, B.~An (eds.)}
\copyrightyear{2023}
\acmYear{2023}
\acmDOI{}
\acmPrice{}
\acmISBN{}

\title[AAMAS-2023 Formatting Instructions]{Offline Multi-Agent Reinforcement Learning \\
with Coupled Value Factorization}

\author{Xiangsen Wang}
\affiliation{
  \institution{Beijing Jiaotong University}
  \city{Beijing}
  \country{China}}
\email{wangxiangsen@bjtu.edu.cn}

\author{Xianyuan Zhan}
\authornote{Corresponding author.}
\affiliation{
  \institution{Tsinghua University}
  \city{Beijing}
  \country{China}}
\email{zhanxianyuan@air.tsinghua.edu.cn}

\begin{abstract}
Offline reinforcement learning (RL) that learns policies from offline datasets without environment interaction has received considerable attention in recent years. Compared with the rich literature in the single-agent case, offline multi-agent RL is still a relatively underexplored area. Most existing methods directly apply offline RL ingredients in the multi-agent setting without fully leveraging the decomposable problem structure, leading to less satisfactory performance in complex tasks. We present OMAC, a new \underline{o}ffline \underline{m}ulti-\underline{a}gent RL algorithm with \underline{c}oupled value factorization. OMAC adopts a coupled value factorization scheme that decomposes the global value function into local and shared components, and also maintains the credit assignment consistency between the state-value and Q-value functions. Moreover, OMAC performs in-sample learning on the decomposed local state-value functions, which implicitly conducts max-Q operation at the local level while avoiding distributional shift caused by evaluating out-of-distribution actions.
Based on the comprehensive evaluations of the offline multi-agent StarCraft II micro-management tasks, we demonstrate the superior performance of OMAC over the state-of-the-art offline multi-agent RL methods.
\end{abstract}


\newcommand{\BibTeX}{\rm B\kern-.05em{\sc i\kern-.025em b}\kern-.08em\TeX}

\begin{document}


\pagestyle{fancy}
\fancyhead{}


\maketitle 


\section{Introduction}

Many real-world scenarios belong to multi-agent systems, such as autonomous vehicle coordination \cite{b3}, network routing \cite{b4}, and power grids \cite{b9}. This gives rise to the active research field of multi-agent RL (MARL) for solving sequential decision-making tasks that involve multiple agents. Although MARL has made some impressive progress in solving complex tasks such as games~\cite{vinyals2019grandmaster, dota}, the successes are mostly restricted to scenarios with high-fidelity simulators or allowing unrestricted interaction with the real environment. In most real-world scenarios, reliable simulators are not available and it can be dangerous or costly for online interaction with the real system during policy learning. The recently emerged offline RL methods provide another promising direction by training the RL agent with pre-collected offline dataset without system interaction~\cite{bcq, bear, kumar2020conservative,zhan2021deepthermal}.

Compared with offline single-agent RL, offline MARL is a relatively underexplored area and considerably more complex~\cite{omar, icq}. Directly incorporating offline RL ingredients into existing MARL frameworks still bears great challenges. Under the offline setting, evaluating value function outside data coverage areas can produce falsely optimistic values, causing the issue of \textit{distributional shift}~\cite{bear}, leading to seriously overestimated value estimates and misguiding policy learning. Hence the key to offline RL is to introduce some form of data-related regularization and learn pessimistically. When adding the multi-agent consideration, the joint action space grows exponentially with the number of agents, the difficulty is further exacerbated. As we need to additionally consider regularizing multi-agent policy optimization with respect to the data distribution, which can be very sparse under high-dimensional joint action space, especially when the offline dataset is small. This can potentially lead to over-conservative multi-agent policies due to extremely limited feasible state-action space under data-related regularization.

Consequently, an effective offline MARL algorithm needs to not only fully leverage the underlying decomposable problem structure, but also organically incorporate offline data-related regularization. Ideally, the data-related regularization should be performed at the individual agent level to avoid the negative impact of sparse data distribution at the joint space, and enable producing a more relaxed yet still valid regularization to prevent distributional shift. Under this rationale, a natural choice is to consider the Centralized Training with Decentralized Execution (CTDE) framework~\cite{b5,b6,vdn} and the Individual-Global-Max (IGM)~\cite{qmix, qtran} condition to decompose the global value function as the combination of local value functions. However, existing offline MARL algorithms that naively combine the CTDE framework with local-level offline RL \cite{icq, omar} still suffer from several drawbacks. First, the value decomposition scheme is not specifically designed for the offline setting. Second, they may still suffer from instability issues caused by bootstrapping error accumulation with coupled offline learning of local value function $Q_i$ and policy $\pi_i$~\cite{bear,iql}. The instability of the local value function will further propagate and negatively impact the learning of the global value function. Applying strong data-related regularization can alleviate the bootstrapping error accumulation during offline learning, but at the cost of over-conservative policy learning and potential performance loss, a common dilemma encountered in many offline RL methods~\cite{doge}. 

To tackle above issues, we propose OMAC, a new \underline{o}ffline \underline{m}ulti-\underline{a}gent RL algorithm with \underline{c}oupled value factorization. OMAC organically marries offline RL with a specially designed coupled multi-agent value decomposition strategy. In additional to decomposing global Q-value function $Q_{tot}$ in typical CTDE framework, OMAC also decomposes $V_{tot}$ into local state-value functions $V_i$ and a shared component $V_{share}$. Moreover, OMAC poses an extra coupled credit assignment scheme between state-value and Q-value functions to enforce consistency and a more regularized global-local relationship. Under this factorization strategy, we can learn an upper expectile local state-value function $V_i$ in a completely in-sample manner. It enables separated learning of the local Q-value function $Q_i$ and the policy $\pi_i$, which improves the learning stability of both the local and global Q-value functions.
We benchmark our method using offline datasets of StarCraft Multi-Agent Challenge (SMAC) tasks \cite{smac}. The results show that OMAC achieves state-of-the-art (SOTA) performance compared with the competing baseline methods. We also conduct further analyses to demonstrate the effectiveness of our design, as well as the sample efficiency of OMAC.


\section{Related Work}
{\bfseries Offline reinforcement learning.} The main challenge in offline RL is to prevent distributional shift and exploitation error accumulation when evaluating the value function on out-of-distribution (OOD) samples. Existing offline RL methods adopt several approaches to regularize policy learning from deviating too much from offline datasets. Policy constraint methods \cite{bcq, bear, brac, awac, doge} add explicit or implicit behavioral constraints to restrain the policy to stay inside the distribution or support of data. Value regularization methods \cite{cql, fisher-brc, cpq} regularize the value function to assign low values on OOD actions. Uncertainty-based and model-based methods~\cite{wu2021uncertainty, bai2021pessimistic,yu2020mopo, zhan2021deepthermal, zhan2021model} estimate the epistemic uncertainty from value functions or learned models to penalize OOD data. Finally, in-sample learning methods \cite{one-step,iql,por} learn the value function entirely within data to avoid directly querying the Q-function on OOD actions produced by policies. The offline RL component of OMAC shares a similar ingredient with in-sample learning methods, which enjoys the benefit of stable and decoupled learning of value functions and policies.

{\bfseries Multi-agent reinforcement learning.} The complexity of multi-agent decision-making problems is reflected in their huge joint action spaces \cite{b8}. In recent years, the CTDE framework~\cite{b5,b6} has become a popular choice to separate agents' learning and execution phases to tackle the exploding action space issue. In CTDE, agents are trained in a centralized manner with global information but learn decentralized policies to make decisions in individual action spaces during execution. Representative MARL algorithms under CTDE framework are the value decomposition methods \cite{vdn, qmix, qtran, qplex}, which decompose the global Q-function into a combination of local Q-functions for scalable multi-agent policy learning.

There have been a few recent attempts to design MARL algorithms under the offline setting. For example, ICQ \cite{icq} uses importance sampling to implicitly constrain policy learning on OOD samples under the CTDE framework. OMAR \cite{omar} extends multi-agent CQL \cite{cql} by adding zeroth-order optimization to avoid policy learning from falling into bad local optima. MADT \cite{madt} leverages the transformer architecture that has strong sequential data modeling capability to solve offline MARL tasks. 

However, the existing offline MARL algorithms simply combine well-established multi-agent frameworks with offline RL ingredients, rather than marry them in an organic way. All of these methods do not fully utilize the underlying decomposable problem structure for offline modeling. Moreover, they rely on the coupled learning process of local Q-functions and policies, which is prone to bootstrapping error and hard to trade-off between policy exploitation and data-related regularization, causing either instability during training or over-conservative policy learning~\cite{iql,doge}. In this work, we develop OMAC to tackle the above limitations of prior works, which enables perfect unification of both multi-agent modeling and offline learning.

\section{Preliminaries}
\subsection{Notations}
A fully cooperative multi-agent task can be described as a decentralized partially observable Markov decision process (Dec-POMDP) \cite{b7}. Dec-POMDP is formally defined by a tuple $G=\langle \mathcal{S}, \mathcal{A}, \mathcal{P}, r, \mathcal{Z}, \mathcal{O}, n, \gamma\rangle$. $s \in \mathcal{S}$ is the true state of the environment. $\mathcal{A}$ denotes the action set for each of the $n$ agents. At every time step, each agent $i \in \{1,2,...n\}$ chooses an action $a_i \in \mathcal{A}$, forming a joint action $\boldsymbol{a}=(a_1,a_2,...a_n) \in \mathcal{A}^n$. It causes a transition to the next state $s^{\prime}$ in the environment according to the transition dynamics $P\left(s^{\prime}|s, \boldsymbol{a}\right): \mathcal{S} \times \mathcal{A}^n \times \mathcal{S} \rightarrow[0,1]$. All agents share the same global reward function $r(s, \boldsymbol{a}): \mathcal{S} \times \mathcal{A}^n \rightarrow \mathbb{R}$. $\gamma \in [0,1)$ is a discount factor. In the partial observable environment, each agent draws an observation $o_i \in \mathcal{O}$ at each step from the observation function $\mathcal{Z}(s, i): \mathcal{S} \times N \rightarrow \mathcal{O}$.  The team of all agents aims to learn a set of policies $\pi=\{\pi_1,\cdots,\pi_n\}$ that maximize their expected discounted returns $\mathbb{E}_{\boldsymbol{a} \in \boldsymbol{\pi}, s \in \mathcal{S}}\left[\sum_{t=0}^{\infty} \gamma^{t} r(s_t, \boldsymbol{a}_t)\right]$. Under the offline setting, we are given a pre-collected dataset $\mathcal{D}$ and the policy learning is conducted entirely with the data samples in $\mathcal{D}$ without any environment interactions.

\begin{figure*}[!t]
\centering
\includegraphics[width=5.1in]{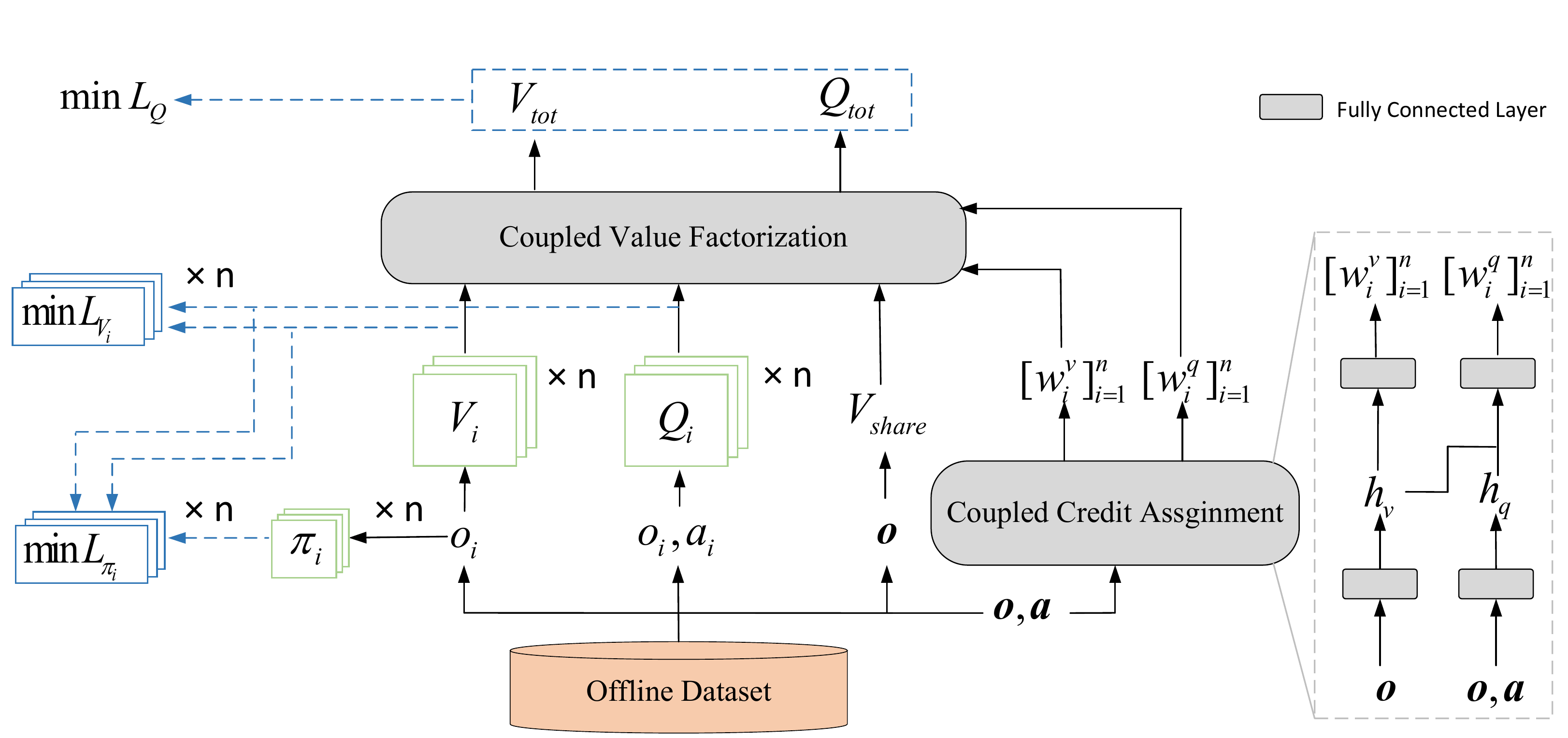}
\caption{The network structure of OMAC.}
\end{figure*}

\subsection{CTDE Framework and IGM Condition}
For multi-agent reinforcement learning, the joint action space increases exponentially with the increase of the number of agents, so it is difficult to query an optimal joint action from the global Q-function $Q_{tot}(\boldsymbol{o}, \boldsymbol{a})$. Besides, applying policy gradient updates with the global Q-function $Q_{tot}$ on the local policy of each agent and treating other agents as part of the environment may lead to poor performance. It suffers from the non-stationarity issue and poor convergence guarantees, as the global Q-function could be affected by the suboptimality of individual agents and potentially disturb policy updates of others~\cite{dop, variance}. 

To address these problems, Centralized Training with Decentralized Execution (CTDE) framework \cite{b5,b6} is proposed. During the training phase, agents have access to the full environment state and can share each other’s experiences. In the execution phase, each agent only selects actions according to its individual observation $o_{i}$. The benefit is that optimization at the individual level leads to optimization of the joint action space, which avoids the exponential growth of the joint action space with the number of agents. 

Under the CTDE framework, Individual-Global-Max (IGM) \cite{qtran} condition is proposed to realize the decomposition of the joint action space, which asserts the consistency between joint and individual greedy action selections in the global and local Q-functions $Q_{tot}$ and $Q_i$:
\begin{equation}
\arg \max _{\boldsymbol{a}} Q_{tot}(\boldsymbol{o}, \boldsymbol{a})=\left(\begin{array}{c}
\arg \max _{a_{1}} Q_{1}\left(o_{1}, a_{1}\right) \\
\vdots \\
\arg \max _{a_{n}} Q_{n}\left(o_{n}, a_{n}\right)
\end{array}\right)
\end{equation}

Through the IGM condition, MARL algorithms can learn the global Q-function and choose actions at the individual level. Meanwhile, it is also possible for offline MARL algorithms to impose constraints on the individual action space, while indirectly enforcing regulations on the joint action space.

\section{Method}
In this section, we formally present the coupled value factorization strategy of OMAC and explain how it can be integrated into effective offline learning.
OMAC decomposes both the state-value and Q-value functions, and also maintains a coupled credit assignment between $Q$ and $V$. With this scheme, OMAC can perform in-sample learning on the decomposed local state-value functions without the involvement of policies, which implicitly enables max-Q operation at the local level while avoiding distributional shift caused by evaluating OOD actions. Finally, the local policy can be separately learned with the well-learned value functions through advantage weighted regression~\cite{awr}.

\subsection{Coupled Value Factorization (CVF)}
\subsubsection{A refined value factorization strategy. }
In OMAC, we consider the following factorization on the global Q-value and state-value. For each agent, we define the local state-value function $V_i$ as the optimal value of the local Q-function $Q_i$. In particular, we decompose the global state-value function into a linear combination of local state-value functions $V_i(o_i)$ with weight function $w_i^v(\boldsymbol{o})$, as well as the shared component based on the full observation $V_{share}(\boldsymbol{o})$. The global Q-function is further decomposed as the state-value function plus a linear combination of local advantages $Q_i(o_i,a_i)-V_i(o_i)$ with weight function $w_i^q(\boldsymbol{o},\boldsymbol{a})$:
\begin{align}
&V_{tot}(\boldsymbol{o}) = \sum_{i=1}^{n} w^v_{i}(\boldsymbol{o}) V_{i}\left({o}_i \right) + V_{share}(\boldsymbol{o}) \notag
\\
&Q_{tot}(\boldsymbol{o}, \boldsymbol{a}) 
= V_{tot}(\boldsymbol{o}) +\sum_{i=1}^{n} {w}^q_{i}(\boldsymbol{o}, \boldsymbol{a}) (Q_{i}\left({o}_i, a_{i}\right)-V_{i}\left(o_i \right)) \notag\\
&V_{i}(o_i) = \underset{a_i} {\max}\;Q_{i}(o_i, a_i), \; {w}^v_{i},  {w}^q_{i} \ge 0,\; \forall i=1,\cdots,n
\label{eq:overall_decomp}
\end{align} 
where we enforce the positivity condition on weight functions $w_i^v(\boldsymbol{o})$ and $w_i^q(\boldsymbol{o},\boldsymbol{a})$. It can be shown that this factorization strategy has a number of attractive characteristics.

\begin{property}
The definition of global Q function in Eq.(\ref{eq:overall_decomp}) satisfies $\underset{\boldsymbol{a}}{\max}\; Q_{tot}(\boldsymbol{o}, \boldsymbol{a}) = V_{tot}(\boldsymbol{o})$ and the IGM condition.
\end{property}
\begin{proof}
Since $V_{i}(o_i) = \underset{a_i}{\max}\; Q_{i}(o_i, a_i)$ for all $a_i$ and $w^q_i \geq 0$, we have $\sum_{i=1}^{n} {w}^q_{i}(\boldsymbol{o}, \boldsymbol{a}) (Q_{i}\left({o}_i, a_{i}\right)-V_{i}\left(o_i \right)) \leq 0$.
Therefore $Q_{tot}(\boldsymbol{o}, \boldsymbol{a}) \leq V_{tot}(\boldsymbol{o})$ and the maximal value of global Q-function  $\underset{\boldsymbol{a}}{\max}\;Q_{tot}(\boldsymbol{o}, \boldsymbol{a}) = V_{tot}(\boldsymbol{o})$ is only achievable when all local Q-functions achieve their maximum (i.e., $\underset{a_i}{\max}\; Q_{i}(o_i, a_i)=V_{i}(o_i)$).
\end{proof}

Second, in Eq. (~\ref{eq:overall_decomp}) the globally shared information is partly captured in the shared component of the state-value function $V_{share}(\boldsymbol{o})$, which is free of the joint actions and not affected by the OOD actions under offline learning. 
The information sharing across agents and credit assignment are captured in weight functions $w^v_{i}(\boldsymbol{o})$, $w^q_{i}(\boldsymbol{o}, \boldsymbol{a})$, and the local value functions $V_i(o_i)$ and $Q_i(o_i,a_i)$ are now only responsible for local observation and action information. The shared and the local information are separated, and agents can make decisions by using local $V_i$ and $Q_i$ at an individual level. As we will show in the later section, this structure also leads to a particularly nice form to incorporate in-sample offline value function learning.

\subsubsection{Coupled credit assignment (CCA)} The value factorization strategy in Eq.~\ref{eq:overall_decomp} can potentially allow too much freedom on the weight function $w^v(\boldsymbol{o})$ and $w^q(\boldsymbol{o}, \boldsymbol{a})$. Ideally, the credit assignment on global state-value and Q-value should be coupled and correlated. Thus, we further design a coupled credit assignment scheme implemented with neural networks to enforce such consistency, which also leads to more regularized relationship between $w^v(\boldsymbol{o})$ and $w^q(\boldsymbol{o}, \boldsymbol{a})$:
\begin{equation}
\begin{aligned}
h_v(\boldsymbol{o}) &= f^{(1)}_v (\boldsymbol{o}),  \quad h_q(\boldsymbol{o}) = f^{(1)}_q(\boldsymbol{o}, \boldsymbol{a}) \\
w^v_i(\boldsymbol{o}) &= |{f^{(2)}_v}(h_v(\boldsymbol{o}))| \\
w^q_i(\boldsymbol{o}, \boldsymbol{a}) &= |f^{(2)}_q(\mathrm{concat}(h_v(\boldsymbol{o}), h_{q}(\boldsymbol{o}, \boldsymbol{a}))|
\end{aligned}
\end{equation}
where $f^{(1)}_v$,  $f^{(2)}_v$, $f^{(1)}_q$, and $f^{(2)}_q$ are hidden neural network layers. We take absolute values on the network outputs to ensure positivity condition of $w^v(\boldsymbol{o})$ and $w^q(\boldsymbol{o}, \boldsymbol{a})$.

CCA enforces a coupled relationship between $w^v(\boldsymbol{o})$ and $w^q(\boldsymbol{o}, \boldsymbol{a})$ by sharing the same observation encoding structure, which makes training on $w^q(\boldsymbol{o}, \boldsymbol{a})$ can also update the parameters of $w^v(\boldsymbol{o})$. This coupling relationship allows more stable credit assignment between state-value and Q-value functions on the same observation $\boldsymbol{o}$. It can also improve data efficiency during training, which is particularly important for the offline setting, since the size of the real-world dataset can be limited.

\subsection{Integrating Offline Value Function Learning}
\subsubsection{Local value function learning.} In the proposed coupled value factorization, the condition of $V_{i}(o_i) = \underset{a_i}{\max}\; Q_{i}(o_i, a_i)$ needs to be forced. Directly implementing this condition can be problematic under the offline setting, as it could lead to queries on OOD actions, causing distributional shift and overestimated value functions. To avoid this issue, one need to instead consider the following condition:
\begin{equation}
V_{i}(o_i) = \underset{ a_i \in \mathcal{A}, \text { s.t. } \pi_{\beta}\left(a_i \mid o_i\right)>0} {\max} Q_{i}(o_i, {a_i}), \label{eq:local_V_cond}
\end{equation}
where $\pi_{\beta}$ is the behavior policy of the offline dataset. Drawing inspiration from offline RL algorithm IQL \cite{iql}, we can implicitly perform the above max-Q operation by leveraging the decomposed state-value functions $V_i$, while also avoiding explicitly learning the behavior policy $\pi_{\beta}$. This can be achieved by learning the local state-value function $V_i(o_i)$ as the upper expectile of target local Q-values $\bar{Q}_i(o_i, a_i)$ based on $(o_i, a_i)$ samples from dataset $\mathcal{D}$. For each agent, its local state-value function $V_i(o_i)$ is updated by minimizing the following objective:
\begin{equation}
L_{V_i}= \mathbb{E}_{(o_i, a_i) \sim \mathcal{D}}\left[L_{2}^{\tau}\left(\bar{Q}_{i}(o_i, a_i)-V_{i}(o_i)\right)\right] \label{eq:local_v},
\end{equation}
where $L_{2}^{\tau}(u)=|\tau-{1}(u<0)| u^{2}$ denotes the expectile regression, which solves an asymmetric least-squares problem given the expectile $\tau\in (0,1)$. When $\tau=0.5$, it reduces to the common least square error. When $\tau \rightarrow 1$, the objective Eq. (\ref{eq:local_v}) makes $V_i(o_i)$ to approximate the maximum of the target local Q-function $\bar{Q}_i(o_i,a_i)$ over actions $a_i$ constrained to the dataset actions.

\subsubsection{Global value function learning. } With the estimated local state-value function $V_i(o_i)$, we can then use it to update the global value functions $V_{tot}$ and $Q_{tot}$, which are essentially parameterized by the shared state-value function $V_{share}(\boldsymbol{o})$, local Q-value function $Q_i(o_i,a_i)$, as well as the credit assignment weight functions $w_i^v(\boldsymbol{o})$ and $w_i^q(\boldsymbol{o}, \boldsymbol{a})$ as in Eq. (\ref{eq:overall_decomp}). These terms can be thus jointly learned by minimizing the following objective:
\begin{equation}
L_{Q}=\mathbb{E}_{\left(\boldsymbol{o}, \boldsymbol{a}, \boldsymbol{o}^{\prime}\right) \sim \mathcal{D}}\left[\left(r(\boldsymbol{o}, \boldsymbol{a})+\gamma V_{tot}\left(\boldsymbol{o}^{\prime}\right)-Q_{tot}(\boldsymbol{o}, \boldsymbol{a}\right))^{2}\right]
\label{eq:q}.
\end{equation}

It should be noted that the learning of both local and global value functions in OMAC is completely performed in an in-sample manner without the involvement of the agent policies $\pi_i$. This separated learning process greatly improves the learning stability of both the local and global value functions, as it avoids querying OOD actions from the policies during Bellman evaluation, which is the main contributor to the distributional shift in offline RL.

\subsection{Local Policy Learning}
Although our method learns the approximated optimal local and global Q-functions, it does not explicitly represent the local policy of each agent for decentralized execution. Therefore a separate policy learning step is needed. With the learned local state-value and Q-value functions $Q_i$ and $V_i$, we can extract the local policies by maximizing the local advantage values with KL-divergence constraints to regularize the policy to stay close to the behavior policy:
\begin{equation}
\begin{aligned}
&\underset{\pi_{i}}{\max }\; {\mathbb{E}}_{{a}_i \sim \pi_i (a_i|o_i)}[Q_i(o_i, a_i)-V_{i}(o_i)]\\
&\text {s.t. }\; D_{\mathrm{KL}}\left(\pi_i(\cdot \mid {o_i}) \| \pi_{\beta, i}(\cdot \mid {o_i})\right) \leq \epsilon
\end{aligned}
\end{equation}

The above optimization problem can be shown equivalent to minimizing the following advantage-weighted regression objective~\cite{awr, awac} by enforcing the KKT condition, which can be solved by sampling directly from the dataset without the need to explicitly learn the local behavior policy $\pi_{\beta, i}$:
\begin{equation}
\begin{aligned}
L_{\pi_i} =\mathbb{E}_{(o_i, a_i) \sim \mathcal{D}}\left[\exp \left(\beta\left(Q_i(o_i, a_i)-V_{i}(o_i)\right)\right) \log \pi_i(a_i|o_i)\right],
\end{aligned} 
\label{eq:policy}
\end{equation}
where $\beta$ is a temperature parameter. For smaller $\beta$ values, the algorithm is more conservative and produces policies closer to behavior cloning. While for larger values, it attempts to recover the maximum of the local Q-function. 

The detailed algorithm of OMAC is summarized below.

\begin{algorithm}[H]
\caption{Pseudocode of OMAC}
\begin{algorithmic}[1]
\REQUIRE{Offline dataset $\mathcal{D}$. hyperparameters $\tau$ and $\beta$.}

\STATE{Initialize local state-value network $V_i$, local Q-value network $Q_i$ and its target network $\bar{Q}_i$, and policy network ${\pi}_i$ for agent $i$=1, 2, ... $n$.}
\STATE{Initialize the shared state-value network $V_{share}$ as well as weight function network $w^v$ and $w^q$.}

\FOR{$t=1, \cdots,$ \emph{max-value-iteration}}
\STATE {Sample batch transitions ${(\boldsymbol{o}, \boldsymbol{a}, r, \boldsymbol{o}^{\prime})}$ from $\mathcal{D}$}

\STATE {Update local state-value function $V_i(o_i)$ for each agent $i$ via Eq. (\ref{eq:local_v}).}

\STATE {Compute $V_{tot}(\boldsymbol{o}^{\prime})$,  $Q_{tot}(\boldsymbol{o}, \boldsymbol{a})$ via Eq. (\ref{eq:overall_decomp}).}

\STATE {Update local Q-value network $Q_i(o_i, a_i)$, weight function network $w^v(\boldsymbol{o})$ and $w^q(\boldsymbol{o},\boldsymbol{a})$ with objective Eq. (\ref{eq:q}).}

\STATE {Soft update target network $\bar{Q}_i(o_i, a_i)$ by ${Q}_i(o_i, a_i)$ for each agent $i$.}

\ENDFOR

\FOR{$t=1, \cdots,$ \emph{max-policy-iteration}}
\STATE {Update local policy network $\pi_i$ for each agent $i$ via Eq. (\ref{eq:policy}).}
\ENDFOR

\end{algorithmic}
\end{algorithm}

\section{Analysis}
\subsection{Optimality Analysis}
In this section, we will show that OMAC can recover the optimal value function under the dataset support constraints. We can show in the following theorem that the learned local and global Q-functions approximate the optimal local and global Q-functions with the data support constraints as the expectile $\tau\rightarrow 1$:
\begin{theorem} \label{theorem:vf}
Given the value factorization strategy in Eq. (\ref{eq:overall_decomp}) and expectile $\tau$, we define $V_i^{\tau}(o_i)$ as the $\tau^{th}$ expectile of $Q_i^{\tau}(o_i, a_i)$ (e.g., $\tau=0.5$ corresponds to the standard expectation) and define $ V^{\tau}_{tot}(\boldsymbol{o}) = \sum_{i=1}^{n} w^v_{i}(\boldsymbol{o}) V^{\tau}_{i}\left({o}_i \right) + V_{share}(\boldsymbol{o})$, then we have
\begin{align}
    \lim _{\tau \rightarrow 1} V_i^{\tau}(o_i)&=\max _{\substack{a_i \in \mathcal{A} \\ \mathrm{ s.t. } \pi_{\beta,i}(a_i)>0}} Q_i^{*}(o_i, a_i)  \label{theorem:local} \\
    \lim _{\tau \rightarrow 1} V_{tot}^{\tau}(\boldsymbol{o})&=\max _{\substack{\boldsymbol{a} \in \mathcal{A}^n \\ \mathrm{ s.t. } \pi_{\beta}(\boldsymbol{a})>0}} Q_{tot}^{*}(\boldsymbol{o}, \boldsymbol{a}) \label{theorem:global}
\end{align}
\end{theorem}

Let $m_{\tau}$ be the $\tau \in(0,1)$ expectile solution to the asymmetric least square problem: $\underset{m_{\tau}}{\arg \min } \mathbb{E}_{x \sim X}\left[L_{2}^{\tau}\left(x-m_{\tau}\right)\right]$. We re-use two lemmas from \citet{iql} related to the expectile properties of a random variable 
to prove Theorem 1:

\begin{lemma} Let $X$ be a random variable with a bounded support and the supremum of the support is $x^*$, then
\begin{equation*}
\begin{aligned}
\lim_{\tau \rightarrow 1} m^{\tau}=x^{*}
\end{aligned}
\end{equation*}
\end{lemma}
The proof is provided in \citet{iql}. It follows as the expectiles of a random variable $X$ have the same supremum $x^*$, and we have $m^{\tau_1} \textless m^{\tau_2}$ for all $\tau_1 \textless \tau_2$. Hence we can obtain the above limit according to the property of bounded monotonically non-decreasing functions.

Let $V_i^{\tau}(o_i)$ be the $\tau$-expectile of $V_i(o_i)$ in OMAC, then we have the following lemma by extending the Lemma 2 of ~\citet{iql} to multi-agent setting:

\begin{lemma}
For all $o_i$, $\tau_1$ and $\tau_2$ such that $\tau_1 \le \tau_2$ we have
$V_i^{\tau_{1}}(o_i) \leq V_i^{\tau_{2}}(o_i)$.
\end{lemma}

\begin{proof}
In OMAC, the learning objective of the local state-value function $V_i(o_i)$, $L_{V_i}= \mathbb{E}_{(o_i, a_i) \sim \mathcal{D}}[L_{2}^{\tau}(\bar{Q}_{i}(o_i, a_i)-V_{i}(o_i))]$ has the same form as the state-value function IQL under the single-agent case. Hence the conclusion of Lemma 1  ($V_{\tau_1}(s)\leq V_{\tau_2}(s)$ for $\forall \tau_1 < \tau_2$) in the IQL paper~\cite{iql} also carries over with the state-value function $V(s)$ being replaced by local state-value functions $V_i(o_i)$ under the multi-agent case. 

\end{proof}

Next, we use the above lemmas to formally prove Theorem 1.

\begin{proof}
We first prove the local part Eq. (\ref{theorem:local}) of Theorem 1. As the local state-value function $V_i(o_i)$ is learned through expectile regression, therefore, for the ${\tau}$-expectile of local state-value $V_i^{\tau}(o_i)$ and an optimal Q-value function constrained to the dataset $Q_i^*(o_i, a_i)$, we have:
\begin{equation}
\begin{aligned}
V_i^{\tau}(o_i) &= \mathbb{E}_{a_i \sim \pi_{\beta, i}(\cdot \mid o_i)}^{\tau}\left[Q^{\tau}_i(o_i, a_i)\right] 
\\
& \leq \max_{\substack{a_i \in \mathcal{A} \\ \mathrm{ s.t. } \pi_{\beta,i}(a_i)>0}} Q^{\tau}_i(o_i, a_i) \leq \max_{\substack{a_i \in \mathcal{A} \\ \mathrm{ s.t. } \pi_{\beta,i}(a_i)>0}} Q^*_i(o_i, a_i)\label{eq:proof_local}
\end{aligned}
\end{equation}
The inequality follows from the fact that the convex combination is smaller than the maximum. 

Thus, $V_i^{\tau}(o_i)$ is a random variable with bounded support and its supremum is $\max _{a_i \in \mathcal{A}_i, \mathrm{ s.t. } \pi_{\beta,i}(a_i)>0} Q^*_i(o_i, a_i)$. Applying Lemma 1, we can obtain the local condition:
\begin{equation*}
\begin{aligned}
\lim _{\tau \rightarrow 1} V_i^{\tau}(o_i)=\max _{\substack{a_i \in \mathcal{A}_i \\ \mathrm{ s.t. } \pi_{\beta,i}(a_i)>0}} Q_i^{*}(o_i, a_i).
\end{aligned}
\end{equation*}
Moreover, based on Lemma 1 and the second inequality in Eq.~(\ref{eq:proof_local}), it's also easy to see:
\begin{equation}
\begin{aligned}
\lim _{\tau \rightarrow 1} Q_i^{\tau}(o_i,a_i)=\max _{\substack{a_i \in \mathcal{A}_i \\ \mathrm{ s.t. } \pi_{\beta,i}(a_i)>0}} Q_i^{*}(o_i, a_i).\label{eq:q_local_lim}
\end{aligned}
\end{equation}

For the global state-value and Q-value functions, according to the couple value factorization strategy in Eq.~(\ref{eq:overall_decomp}), we have:
\begin{displaymath}
\begin{aligned}
    &Q_{tot}^{\tau}(\boldsymbol{o}, \boldsymbol{a}) = V_{tot}^{\tau}(\boldsymbol{o}) +\sum_{i=1}^{n} {w}^q_{i}(\boldsymbol{o}, \boldsymbol{a}) (Q_{i}^{\tau}\left({o}_i, a_{i}\right)-V_{i}^{\tau}\left(o_i \right)) \\
    &= \sum_{i=1}^{n} w^v_{i}(\boldsymbol{o}) V_{i}^{\tau}\left({o}_i \right) + V_{share}(\boldsymbol{o}) +\sum_{i=1}^{n} {w}^q_{i}(\boldsymbol{o}, \boldsymbol{a}) (Q_{i}^{\tau}\left({o}_i, a_{i}\right)-V_{i}^{\tau}\left(o_i \right))
\end{aligned}
\end{displaymath}
Taking the limit $\tau\rightarrow 1$ on both sides, and use the local condition and Eq.~(\ref{eq:q_local_lim}), we have:
\begin{displaymath}
\begin{aligned}
    &\lim _{\tau \rightarrow 1}Q_{tot}^{\tau}(\boldsymbol{o}, \boldsymbol{a}) 
    = \sum_{i=1}^{n} w^v_{i}(\boldsymbol{o}) \max _{\substack{a_i \in \mathcal{A}_i \\ \mathrm{ s.t. } \pi_{\beta,i}(a_i)>0}} Q_i^{*}(o_i, a_i) + V_{share}(\boldsymbol{o}) \\
    &+\sum_{i=1}^{n} {w}^q_{i}(\boldsymbol{o}, \boldsymbol{a}) \left(\max _{\substack{a_i \in \mathcal{A}_i \\ \mathrm{ s.t. } \pi_{\beta,i}(a_i)>0}} Q_i^{*}(o_i, a_i)-\max _{\substack{a_i \in \mathcal{A}_i \\ \mathrm{ s.t. } \pi_{\beta,i}(a_i)>0}} Q_i^{*}(o_i, a_i)\right) \\
    &= \sum_{i=1}^{n} w^v_{i}(\boldsymbol{o}) \max_{\substack{a_i \in \mathcal{A} \\ \mathrm{ s.t. } \pi_{\beta,i}(a_i)>0}} Q^*_i(o_i, a_i)  + V_{share}(\boldsymbol{o})  \\
    &= \max _{\substack{\boldsymbol{a} \in \mathcal{A}^n \\ \mathrm{ s.t. } \pi_{\beta}(\boldsymbol{a})>0}} Q_{tot}^{*}(\boldsymbol{o}, \boldsymbol{a})
\end{aligned}
\end{displaymath}
On the other hand, we have:
\begin{displaymath}
\begin{aligned}
    V_{tot}^{\tau}(\boldsymbol{o}) &= \sum_{i=1}^{n} w^v_{i}(\boldsymbol{o}) V_{i}^{\tau}\left({o}_i \right) + V_{share}(\boldsymbol{o}) \\
    &= \sum_{i=1}^{n} w^v_{i}(\boldsymbol{o}) \mathbb{E}_{a_i \sim \pi_{\beta, i}(\cdot \mid o_i)}^{\tau}\left[Q^{\tau}_i(o_i, a_i)\right]  + V_{share}(\boldsymbol{o}) \\
    &\leq \sum_{i=1}^{n} w^v_{i}(\boldsymbol{o}) \max_{\substack{a_i \in \mathcal{A} \\ \mathrm{ s.t. } \pi_{\beta,i}(a_i)>0}} Q^*_i(o_i, a_i)  + V_{share}(\boldsymbol{o}) \\
    &= \max _{\substack{\boldsymbol{a} \in \mathcal{A}^n \\ \mathrm{ s.t. } \pi_{\beta}(\boldsymbol{a})>0}} Q_{tot}^{*}(\boldsymbol{o}, \boldsymbol{a})
\end{aligned}
\end{displaymath}
Thus $V_{tot}^{\tau}(\boldsymbol{o})$ has bounded support with the supremum given above. Applying Lemma 1, we obtain the global part of Theorem 1:
\begin{equation*}
\begin{aligned}
    \lim _{\tau \rightarrow 1} V_{tot}^{\tau}(\boldsymbol{o})&=\max _{\substack{\boldsymbol{a} \in \mathcal{A}^n \\ \mathrm{ s.t. }\; \pi_{\beta}(\boldsymbol{a})>0}} Q_{tot}^{*}(\boldsymbol{o}, \boldsymbol{a})
\end{aligned}
\end{equation*}

\end{proof}

\begin{figure*} [t]
\centering
\begin{minipage}[t]{0.999\linewidth}
\centering
\includegraphics[width=1.7in]{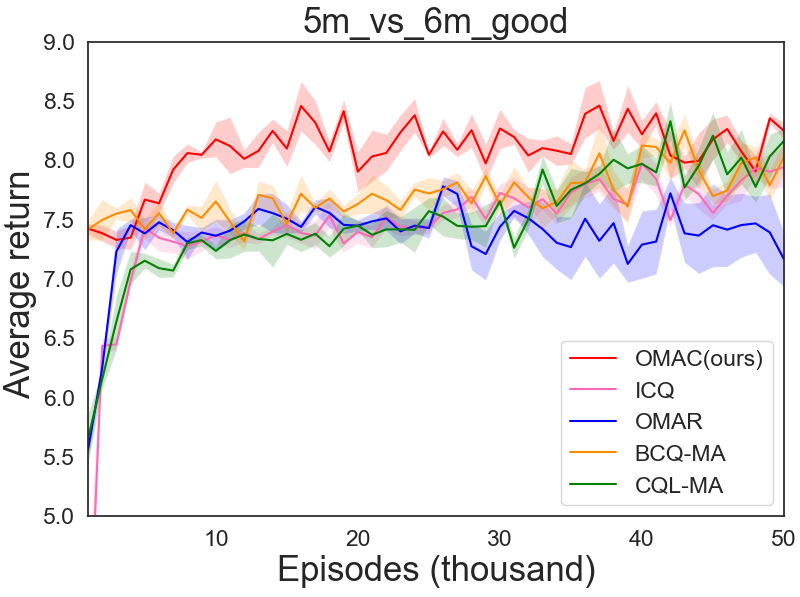}\hspace{0pt}
\includegraphics[width=1.7in]{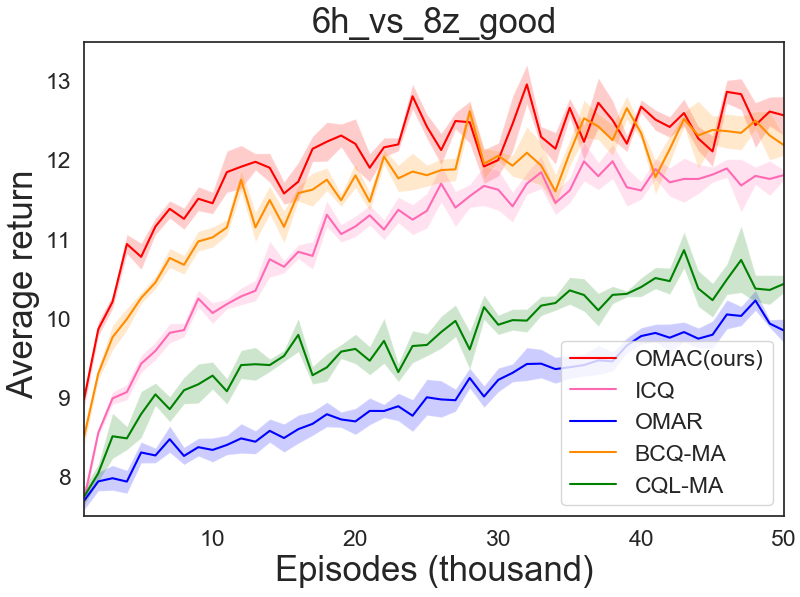}\hspace{0pt}
\includegraphics[width=1.7in]{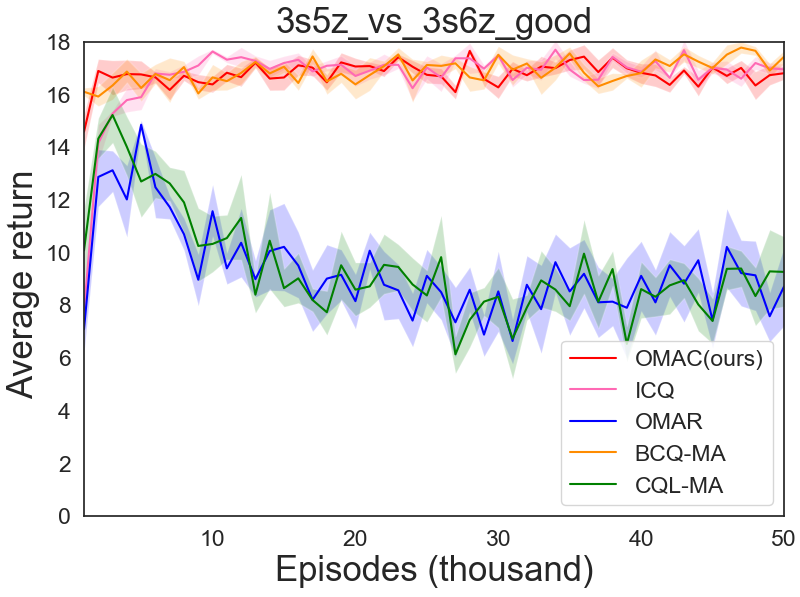}\hspace{0pt}
\includegraphics[width=1.7in]{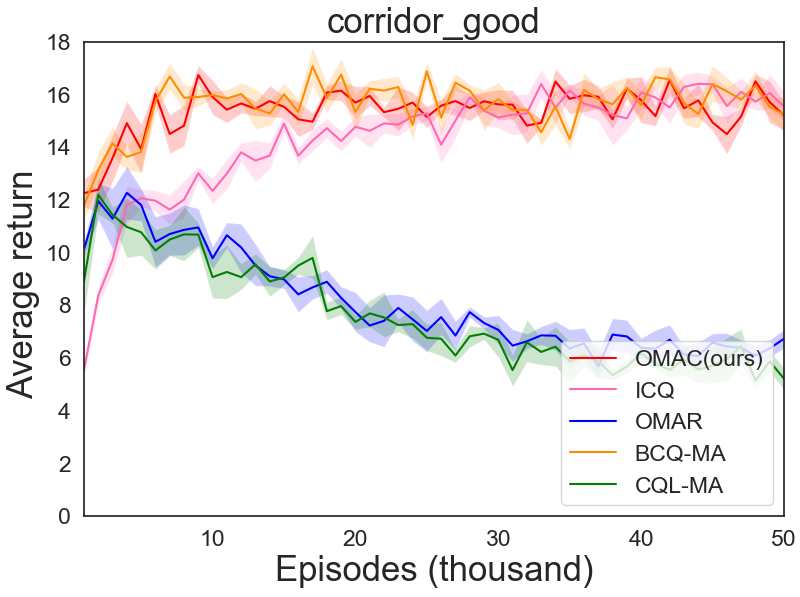}\hspace{0pt}
\includegraphics[width=1.7in]{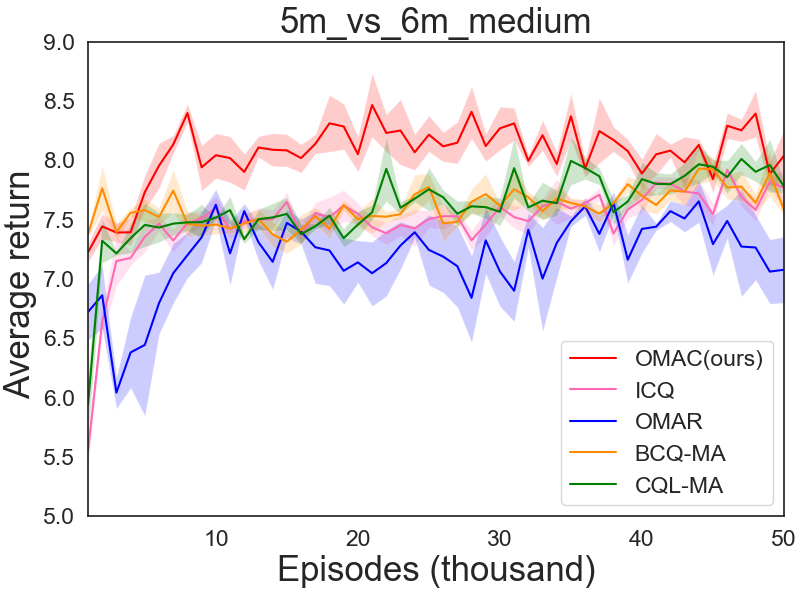}\hspace{0pt}
\includegraphics[width=1.7in]{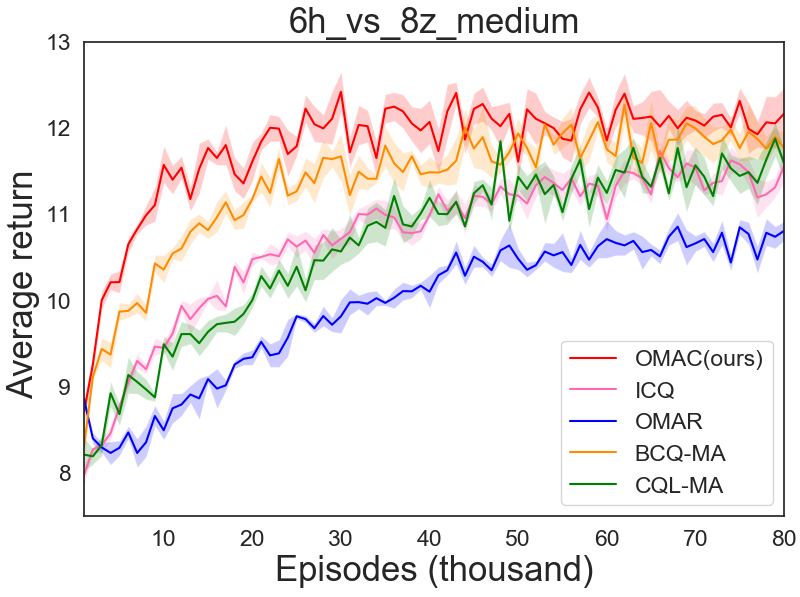}\hspace{0pt}
\includegraphics[width=1.7in]{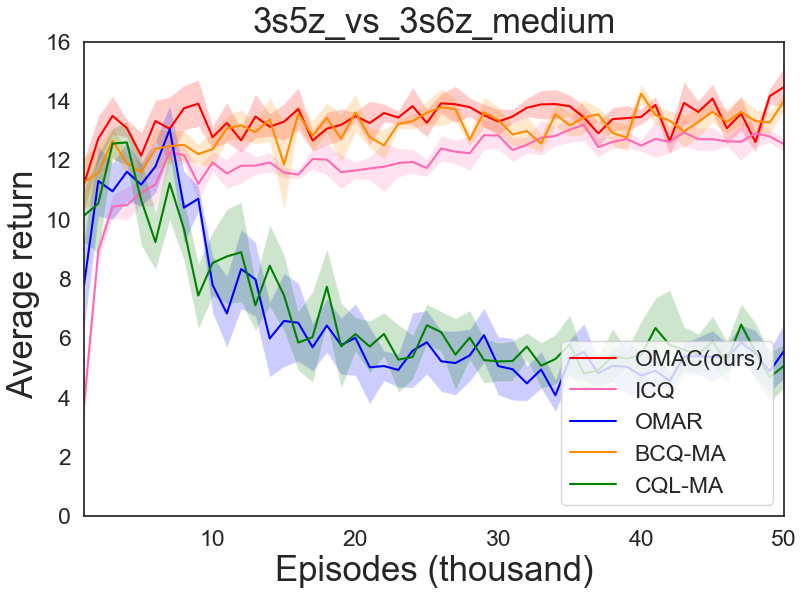}\hspace{0pt}
\includegraphics[width=1.7in]{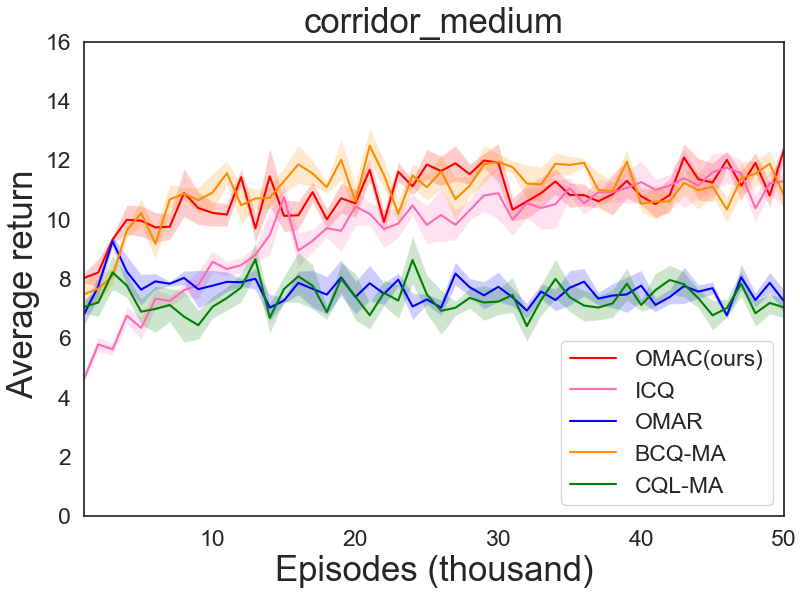}\hspace{0pt}
\includegraphics[width=1.7in]{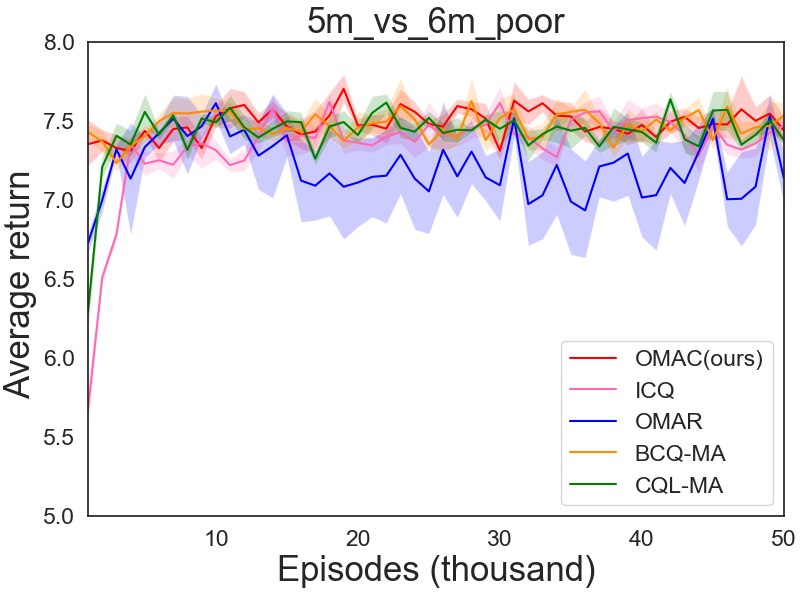}\hspace{0pt}
\includegraphics[width=1.7in]{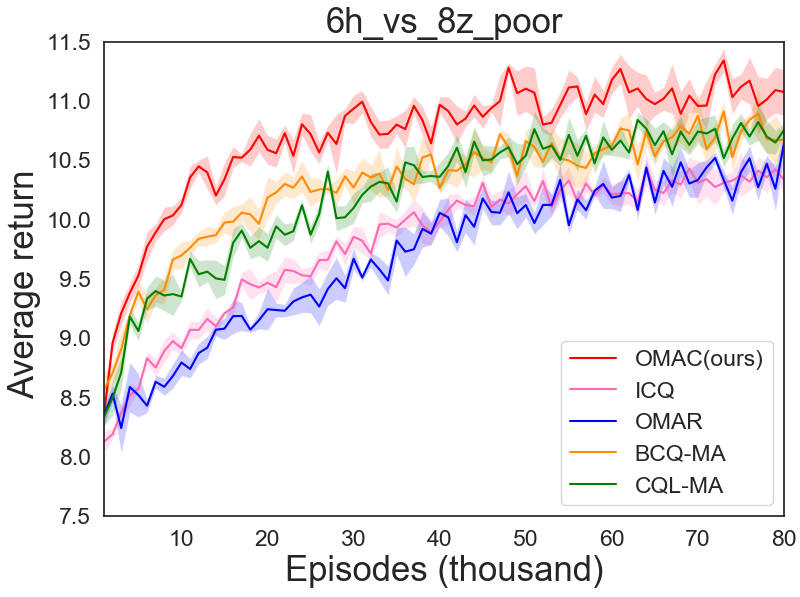}\hspace{0pt}
\includegraphics[width=1.7in]{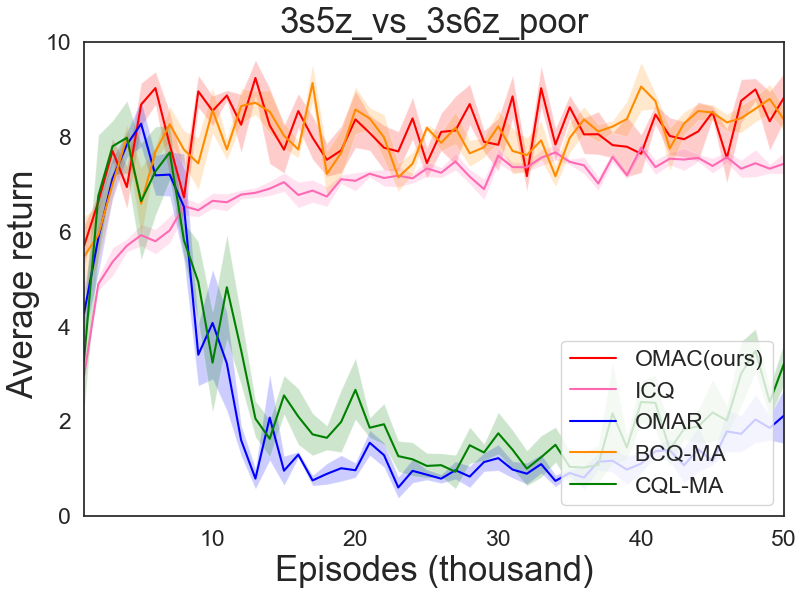}\hspace{0pt}
\includegraphics[width=1.7in]{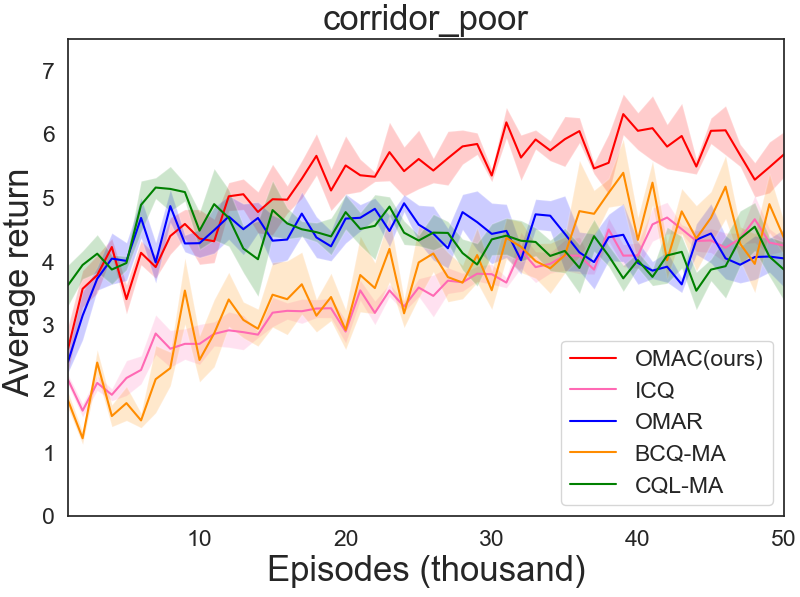}\hspace{0pt}
\end{minipage}%
\centering
\caption{Learning curves over 5 random seeds on the offline SMAC tasks.}

\label{fig:learning_curves}
\end{figure*}

\section{Experiments}
In this section, we present the experiment setups and comparative evaluation of OMAC against SOTA offline MARL baseline algorithms. We also carry out data efficiency analyses as well as ablation studies to better illustrate the effectiveness and properties of our algorithm.

\subsection{Experiment Settings}
\subsubsection{Offline datasets.}
We choose the StarCraft Multi-Agent Challenge (SMAC) benchmark \cite{smac} as our testing environment. SMAC is a popular multi-agent cooperative control environment for evaluating advanced MARL methods due to its high control complexity. It focuses on micromanagement challenges where a group of algorithm-controlled agents learns to defeat another group of enemy units controlled by built-in heuristic rules, and the goal is to maximize the average return to achieve victory.

The offline SMAC dataset used in this study is provided by~\cite{madt}, which is the largest open offline dataset on SMAC. Different from single-agent offline datasets, it considers the property of Dec-POMDP, which owns local observations and available actions for each agent. The dataset is collected from the trained MAPPO agent~\cite{mappo2}, and includes three quality levels: good, medium, and poor. SMAC consists of several StarCraft II multi-agent micromanagement maps. We consider 4 representative battle maps, including 1 hard map (5m\_vs\_6m), and 3 super hard maps (6h\_vs\_8z, 3s5z\_vs\_3s6z, corridor).

\subsubsection{Baselines.} We compare OMAC against four recent offline MARL algorithms: ICQ \cite{icq}, OMAR \cite{omar}, multi-agent version of BCQ \cite{bcq} and CQL \cite{cql}, namely BCQ-MA and CQL-MA. BCQ-MA and CQL-MA use linear weighted value decomposition structure for the multi-agent setting. Details for baseline implementations and hyperparameters in OMAC are discussed in Appendix. 

\subsection{Comparative Results}
We report the mean and standard deviation of average returns for the offline SMAC tasks during training in Fig. \ref{fig:learning_curves}. Each algorithm is evaluated using 32 independent episodes and run with 5 random seeds. The results show that OMAC consistently outperforms all baselines and achieves state-of-the-art performance in most maps. For the super hard SMAC map such as 6h\_vs\_8z or corridor, the cooperative relationship of agents is very complex and it is difficult to learn an accurate global Q-value function. Due to the couple value factorization, the global $Q_{tot}$ of OMAC has stronger expressive capability, which makes OMAC have better performance than other baseline algorithms. Moreover, both the local and
global value functions in OMAC are completely performed in an in-sample manner without the involvement of the agent policies $\pi_i$, which also leads to better offline performance.

\begin{table*}[t]
\centering
\begin{tabular}{lll|ccccc}
\hline
Map   & Dataset  &Ratio & OMAC(ours)   & ICQ        & OMAR       & BCQ-MA             & CQL-MA     \\ \hline
6h\_vs\_8z     & good  &100\%   & \textbf{12.57±0.47} & 11.81±0.12 & 9.85±0.28  & 12.19±0.23         & 10.44±0.20 \\
6h\_vs\_8z     & good  &50\%   & \textbf{12.28±0.43} & 11.59±0.43  & 9.00±0.27 & 11.93±0.52         & 9.06±0.38 \\
6h\_vs\_8z     & good  &10\%   & \textbf{10.61±0.18} & 8.86±0.21  & 7.88±0.19 & 9.92±0.10         & 8.41±0.16 \\
\hline
6h\_vs\_8z     & medium  &100\%  & \textbf{12.17±0.52} & 11.56±0.34 & 10.81±0.21 & 11.77±0.36         & 11.59±0.35 \\
6h\_vs\_8z     & medium    &50\% & \textbf{11.98±0.32} & 10.80±0.25 & 10.047±0.11 & 11.51±0.33         & 10.68±0.23 \\
6h\_vs\_8z     & medium  &10\%  & \textbf{10.86±0.08} & 9.47±0.27  &8.27±0.07 & 9.92 ±0.17        & 8.61±0.29 \\
\hline
6h\_vs\_8z     & poor  &100\%  & \textbf{11.08±0.36} & 10.34±0.23 & 10.64±0.20 & 10.67±0.19         & 10.76±0.11  \\
6h\_vs\_8z     & poor    &50\% & \textbf{10.84±0.14} & 9.97±0.14  & 9.87±0.32  & 10.39±0.46         & 9.99±0.44 \\
6h\_vs\_8z     & poor  &10\%  &\textbf{8.59±0.21} & 7.51±0.22  &7.29±0.09 & 8.34±0.19    & 8.18±0.39 \\ 
\hline
\end{tabular}
    \caption{Evaluation on data efficiency of different methods on offline SMAC datasets with reduced size}
    \label{tab:sample_efficiency}
\end{table*}

\begin{figure*}[t]
	\centering
	\subfloat[Ablation on value decomposition]{\includegraphics[width=2.7in]{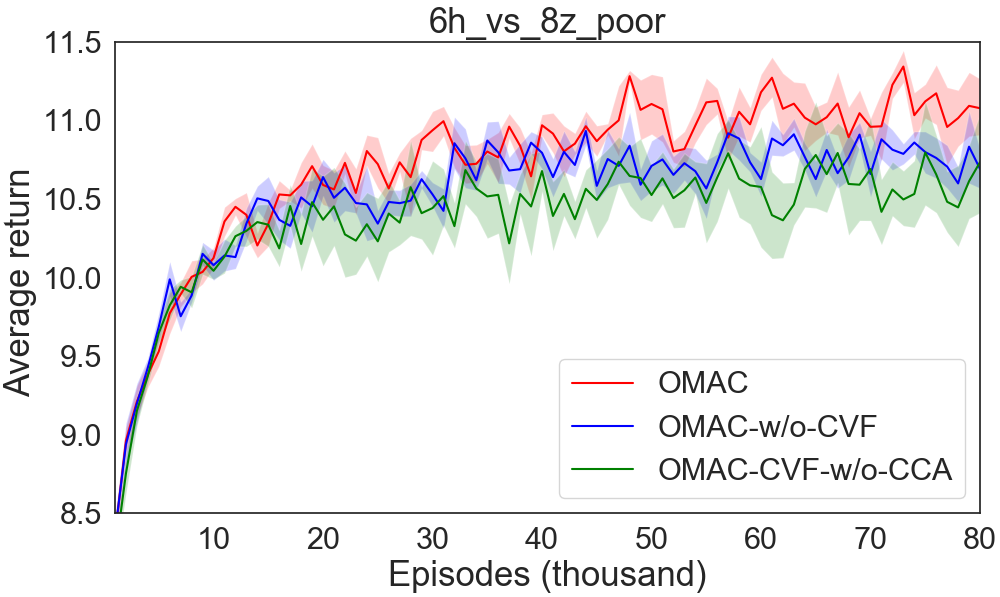}}\quad \quad
\subfloat[Comparison on implicit max-Q operation]{\includegraphics[width=2.7in]{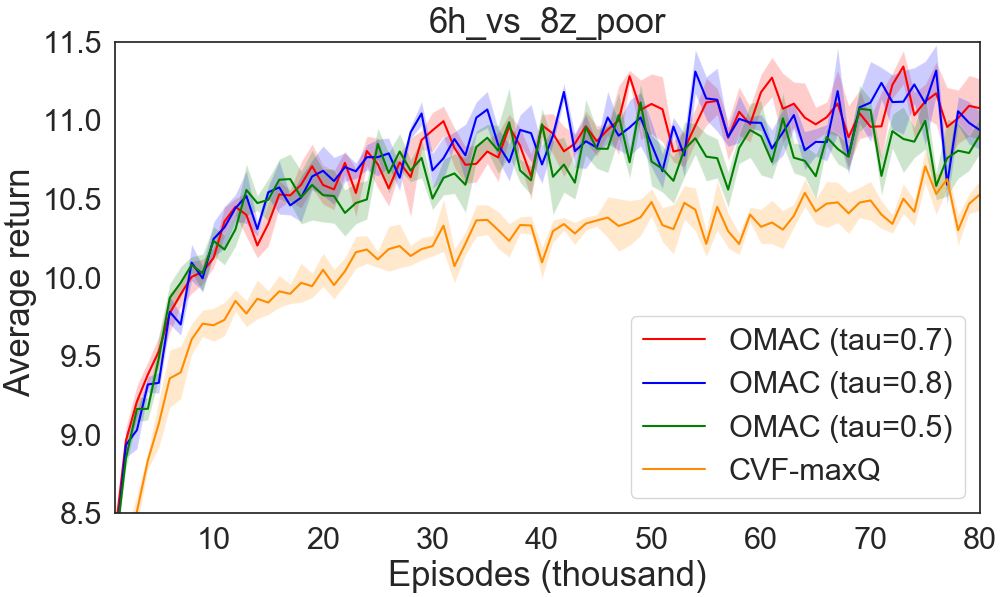}}
	\caption{Analyses and ablations on the design components of OMAC
	}
	\label{fig:ablation}
\end{figure*}

\subsection{Evaluation on Data Efficiency}
Data efficiency is particularly important for offline RL applications, as real-world datasets can often be limited. For offline MARL problems, this can be even more challenging due to high-dimensional joint state and action spaces, which potentially requires a larger amount of data to guarantee reasonable model performance.
To demonstrate the sample utilization efficiency of OMAC over baseline algorithms, we further conduct experiments on SMAC map 6h\_vs\_8z with the size of the original datasets reduced to 50\% and 10\%.
As shown in Table \ref{tab:sample_efficiency}, OMAC consistently outperforms the baseline algorithms in all tasks. Moreover, it is observed that OMAC experiences a lower level of performance drop when the dataset size is reduced, whereas recent offline MARL counterpart algorithms like ICQ and OMAR suffer from noticeable performance drop.
The reasons for the better data efficiency of OMAC could be due to the use of both coupled credit assignment and in-sample learning. As in OMAC, training on the credit assignment weights $w^q$ also updates the parameters of $w^v$, which enables effective re-use of data. Meanwhile, the local state-value function $V_i$ is learned by expectile regression in a supervised manner rather than performing dynamic programming, which in principle can be more stable and sample efficient.

\subsection {Analyses on the Design Components of OMAC}

In this section, we conduct ablation studies and additional analyses to examine the effectiveness of different design components of OMAC.

\subsubsection{Ablation on coupled value factorization. }
To examine the impact of our coupled value factorization (CVF) strategy, we conduct the ablation study on map 6h\_vs\_8z with poor dataset. We test OMAC and the variant without using the coupled value factorization (OMAC-w/o-CVF), which uses the linear weighted decomposition structure used by ICQ and OMAR. As shown in Fig. \ref{fig:ablation}(a), OMAC performs better than OMAC-w/o-CVF, which clearly suggests the advantage of coupled value factorization strategy.

\begin{figure}[t]
\centering
\includegraphics[width=2.5in]{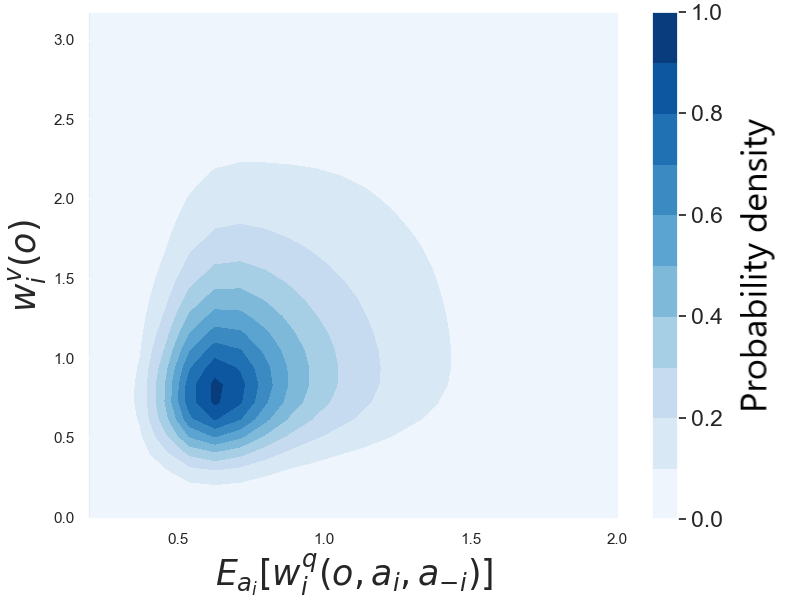}
\caption{Probability density plot of $w^v_i(\boldsymbol{o})$ and $\mathbb{E}_{{a}_i} [w_i^q(\boldsymbol{o}, a_i,  \boldsymbol{a}_{-i})]$ using credit assignment weight functions $w^v$ and $w^q$ learned in the 6z\_vs\_8z\_poor task.}
\label{fig:couple}
\end{figure}

\subsubsection{Analyses on coupled credit assignment}
An important design in our method is the coupled credit assignment (CCA) scheme in CVF that learns $w^v$ and $w^q$ dependently. 
We compare OMAC and the variant without coupled credit assignment (OMAC-CVF-w/o-CCA), which is trained by implementing $w^v$ and $w^q$ as two independent networks without the coupled structure. Fig. \ref{fig:ablation}(a) also shows that OMAC has better performance and stability than OMAC-CVF-w/o-CCA, suggesting the important contribution of CCA on model performance due to more stable and better regularized credit assignment between the state-value function and Q-value function on the same observations.

To verify that the coupled network structure of $w^v(\boldsymbol{o})$ and $w^q(\boldsymbol{o},\boldsymbol{a})$ in OMAC indeed produces correlated credit assignments on state-value and Q-value functions on the same observations, we further conduct an experiment to inspect their produced values.
We use the learned models of $w^v(\boldsymbol{o})$ and $w^q(\boldsymbol{o},\boldsymbol{a})$ to evaluate $\mathbb{E}_{{a}_i} [w_i^q(\boldsymbol{o}, a_i,  \boldsymbol{a}_{-i})]$ and its relationship with $w^v_i(\boldsymbol{o})$ for an arbitrary agent $i$. Based on the results plotted in Fig. \ref{fig:couple}, we observe that a positively correlated relationship exists between $w^v(\boldsymbol{o})$ and $w^q(\boldsymbol{o},\boldsymbol{a})$.

\subsubsection{Impact of implicit max-Q operation on local value functions. }
In our method, OMAC implicitly performs the max-Q operation at the local level by learning the local state-value function $V_i(o_i)$ as the upper expectile of ${Q}_i(o_i, a_i)$ based on $(o_i, a_i)$ samples entirely from dataset $\mathcal{D}$. Choosing a appropriate $\tau$ will make $V_i(o_i)$ to approximate the maximum of the local Q-function ${Q}_i(o_i,a_i)$ over actions $a_i$ constrained to the dataset actions.

Fig. \ref{fig:ablation} (b) shows the performance of OMAC with different levels of $\tau$ on the offline 6h\_vs\_8z\_poor dataset. The performances of OMAC ($\tau=0.7$) and OMAC ($\tau=0.8$) are close and are much better than OMAC ($\tau=0.5$), showing some degree of hyperparameter robustness when $\tau$ is reasonably large. With $\tau=0.5$, the local state-value function $V_i$ is essentially learned to be the expected value of $Q_i$, which is more conservative and leads to suboptimal performance.
To illustrate the benefit of the implicit max-Q evaluation in OMAC, we also implement another algorithm CVF-maxQ for comparison. In CVF-maxQ, the coupled value factorization structure is partially preserved while all the local state-value $V_i(o_i)$ are replaced by $\max_{a_i} Q_{i}(o_i, {a_i})$. As Fig. \ref{fig:ablation} (b) shows, OMAC performs much better than CVF-maxQ, which clearly suggests the advantage of performing local max-Q operation in an implicit and in-sample manner under offline learning.

\section {Conclusion}
In this paper, we propose a new offline MARL algorithm named OMAC. OMAC adopts a coupled value factorization structure, which organically marries offline RL with a specially designed coupled multi-agent value decomposition strategy and has stronger expressive capability. Moreover, OMAC performs in-sample learning on the decomposed local state-value functions, which implicitly conducts the max-Q operation at the local level while avoiding distributional shift caused by evaluating on the out-of-distribution actions. We benchmark our method using offline datasets of SMAC tasks and the results show that OMAC achieves superior performance and better data efficiency over the state-of-the-art offline MARL methods.

\bibliographystyle{ACM-Reference-Format} 
\bibliography{aamas23.bib}

\clearpage

\appendix

\section{Implementation Details}\label{app:imp}

\subsection{Details of OMAC}
In this paper, all experiments are implemented with pytorch and executed on NVIDIA V100 GPUs. The local value networks and policy networks are represented by 3-layers ReLU activated MLPs with 256 units for each hidden layer. For the CCA weight network, we use 3-layer ReLU activated MLPs with 64 units for each hidden layer. All the networks are optimized by Adam optimizer.

\subsection{Details of baselines}
We compare OMAC against four recent offline MARL algorithms: ICQ \cite{icq}, OMAR \cite{omar}, BCQ-MA and CQL-MA. For the ICQ and OMAR, we implement them based on the algorithm description in their papers. BCQ-MA is the multi-agent version of BCQ, and CQL-MA is the multi-agent version of CQL. BCQ-MA and CQL-MA use linear weighted value decomposition structure as $Q_{tot}=\sum_{i=1}^{n} {w_i}(\boldsymbol{o}) Q_{i}\left({o}_i, a_{i}\right) + b(\boldsymbol{o})$ for the multi-agent setting. The policy constrain of BCQ-MA and the value regularization of CQL-MA are both imposed on the local Q-value.

\subsection{Hyperparameters}
The hyperparameters of OMAC and baselines are listed in Table \ref{hyper-param}. The two important hyperparameters of OMAC are expectile parameter $\tau$ and temperature parameter $\beta$. 
The expectile parameter $\tau$ makes $V_i(o_i)$ approximate the maximum of the target local Q-function $\bar{Q}_i(o_i,a_i)$ over actions $a_i$ constrained to the dataset actions. Considering the stability and sample efficiency, we use $\tau=0.7$ to make $V_i$ to approximate the maximum of the local $Q_i$. 

The temperature parameter $\beta$ is used for the policy learning. For lower $\beta$ values, the algorithm is more conservative. While for higher values, it attempts to approximate the maximum of the local Q-function. We use $\beta=0.5$ on the good and medium datasets of 3s5z\_vs\_3s6z and corridor map. On these maps, the quality of the data set is relatively high, so we choose a lower $\beta$. On other datasets, we use $\beta=1.0$ to approximate the maximum of the local Q-function.

\begin{table}[htbp]
\centering
\begin{tabular}{ll}
\hline
Hyperparameter                 & Value \\
\hline
\multicolumn{2}{c}{\textbf{Shared parameters}} \\
Value network learning rate     & 5e-4  \\
Policy network learning rate    & 5e-4  \\
Optimizer                       & Adam  \\
Target update rate              & 0.005 \\
Batch size                      & 128   \\
Discount factor                 & 0.99  \\
Hidden dimension                & 256   \\
\hline
\multicolumn{2}{c}{\textbf{OMAC}}         \\
CCA network hidden dimension & 64    \\
Expectile parameter $\tau$      & 0.7   \\
Temperature parameter $\beta$      & 1 or 0.5   \\
\hline
\multicolumn{2}{c}{\textbf{Others}}             \\
Weight network hidden dimension & 64    \\
Threshold  (BCQ-MA)                     & 0.3 \\
$\alpha$   (OMAR, CQL-MA)                   & 1.0   \\

\hline
\end{tabular}
  \caption{Hyper-parameter of OMAC and baselines}
  \label{hyper-param}
\end{table}

\section{Experiment Settings}\label{app:exp}
We choose the StarCraft Multi-Agent Challenge (SMAC) benchmark as our testing environment. SMAC is a popular multi-agent cooperative control environment for evaluating advanced MARL methods due to its high control complexity. SMAC consists of a set of StarCraft II micro scenarios, and all scenarios are confrontations between two groups of units. The units of the first group are controlled by agents based on MARL algorithm, while the units of another group are controlled by a built-in heuristic game AI bot with different difficulties. The initial location, quantity and type of units, and the elevated or impassable terrain vary from scenario to scenario. The available actions for each agent include no operation, move [direction], attack [enemy id], and stop. The reward that each agent receives is the same. Agents receive a joint reward calculated from the hit-point damage dealt and received.

The offline SMAC dataset used in this study is provided by~\citet{madt}, which is the largest open offline dataset on SMAC. Different from single-agent offline datasets, it considers the property of Dec-POMDP, which owns local observations and available actions for each agent. The dataset is collected from the trained MAPPO agent, and includes three quality levels: good, medium, and poor. SMAC consists of several StarCraft II multi-agent micromanagement maps. We consider 4 representative battle maps, including 1 hard map (5m\_vs\_6m), and 3 super hard maps (6h\_vs\_8z, 3s5z\_vs\_3s6z, corridor), as our experiment tasks. The task types of the maps are listed in the Table \ref{maps}. For each original large dataset, we randomly sample 1000 episodes as our dataset.

\begin{table}[htbp]
	\centering
	\setlength{\tabcolsep}{4mm}
    \begin{tabular}{cc}
\hline
Map Name     & Type              
\\  \hline  
5m\_vs\_6m   &  homogeneous \& asymmetric \\
6h\_vs\_8z   & micro-trick: focus fire    \\
3s5z\_vs\_3s6z & heterogeneous \& asymmetric      \\ 
corridor    & micro-trick: wall off     \\\hline
\end{tabular}
  \caption{SMAC maps for experiments.}
  \label{maps}
\end{table}

\section{Additional Results}\label{app:exp}
\subsection{Results of offline SMAC tasks}
We report the mean and standard deviation of average returns for the offline SMAC tasks in Table \ref{tab:bench}. The datasets on each map include three quality levels: good, medium, and poor. We report the mean and standard deviation of average returns for the offline SMAC tasks during training. Each algorithm is evaluated using 32 independent episodes and run with 5 random seeds.

\begin{table*}[t]
\centering
\begin{tabular}{ll|ccccc}
\hline
Map            & Dataset & OMAC(ours)          & ICQ        & OMAR       & BCQ-MA             & CQL-MA     \\ \hline
5m\_vs\_6m     & good    & \textbf{8.25±0.12}  & 7.94±0.32  & 7.17±0.42  & 8.03±0.31          & 8.17±0.20  \\
5m\_vs\_6m     & medium  & \textbf{8.04±0.42}  & 7.77±0.30  & 7.08±0.51  & 7.58±0.10          & 7.78±0.10  \\
5m\_vs\_6m     & poor    & 7.44±0.16           & 7.47±0.13  & 7.13±0.30   & \textbf{7.53±0.15} & 7.38±0.06  \\ \hline
6h\_vs\_8z     & good    & \textbf{12.57±0.47} & 11.81±0.12 & 9.85±0.28  & 12.19±0.23         & 10.44±0.20 \\
6h\_vs\_8z     & medium  & \textbf{12.17±0.52} & 11.56±0.34 & 10.81±0.21 & 11.77±0.36         & 11.59±0.35 \\
6h\_vs\_8z     & poor    & \textbf{11.08±0.36} & 10.34±0.23 & 10.64±0.20 & 10.67±0.19         & 10.76±0.11 \\ \hline
3s5z\_vs\_3s6z & good    & 16.81±0.46          & 16.95±0.39 & 8.71±2.84  & \textbf{17.43±0.46}         & 9.27±2.53  \\
3s5z\_vs\_3s6z & medium  & \textbf{14.47±1.11} & 12.55±0.53 & 5.58±1.77  & 13.99±0.62         & 5.08±1.45  \\
3s5z\_vs\_3s6z & poor    & \textbf{8.82±0.95}  & 7.43±0.42  & 2.12±1.07  & 8.36±0.45          & 3.22±0.87  \\ \hline
corridor       & good    & 15.21±1.06          & \textbf{15.55±1.13} & 6.74±0.69  & 15.24±1.21         & 5.22±0.81  \\
corridor       & medium  & \textbf{12.37±0.51} & 11.30±1.57 & 7.26±0.71  & 10.82±0.92         & 7.04±0.66  \\
corridor       & poor    & \textbf{5.68±0.65}  & 4.25±0.17  & 4.05±0.86  & 4.37±0.57          & 3.89±0.89 \\
\hline
\end{tabular}
    \caption{Average scores and standard deviations over 5 random seeds on the offline SMAC tasks}
    \label{tab:bench}
\end{table*}

\begin{figure*} [t]
\centering
\begin{minipage}[t]{0.999\linewidth}
\centering
\includegraphics[width=2.3in]{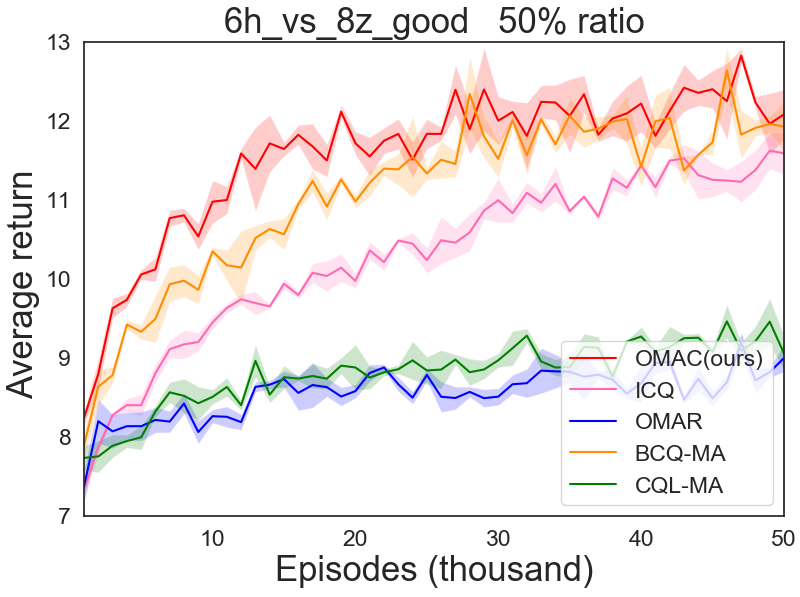}\hspace{0pt}
\includegraphics[width=2.3in]{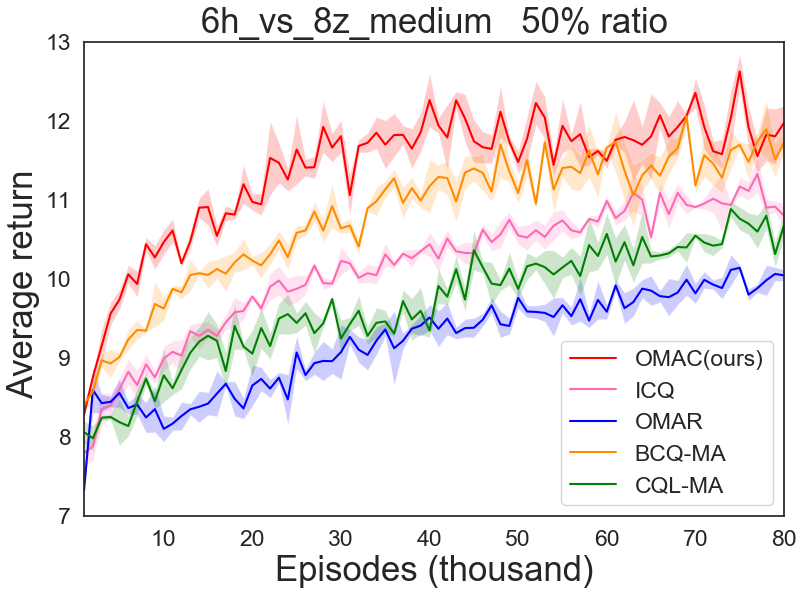}\hspace{0pt}
\includegraphics[width=2.3in]{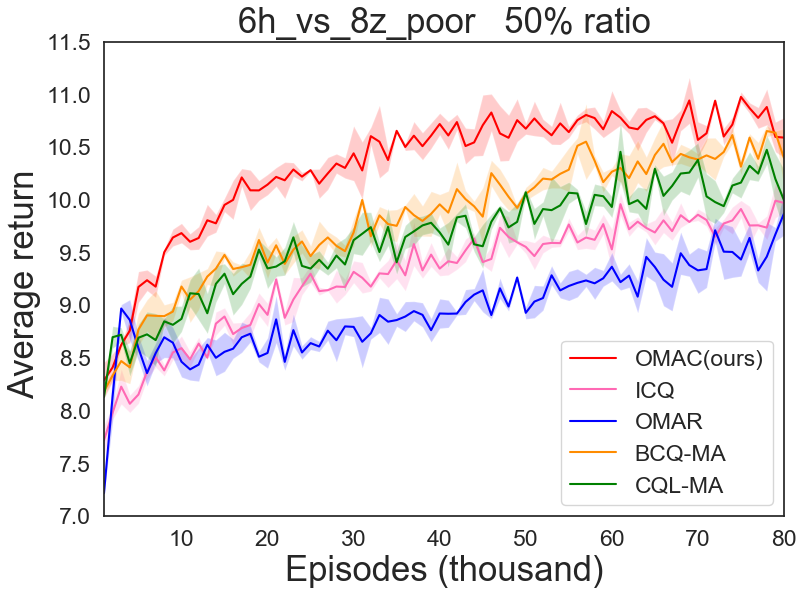}\hspace{0pt}
\includegraphics[width=2.3in]{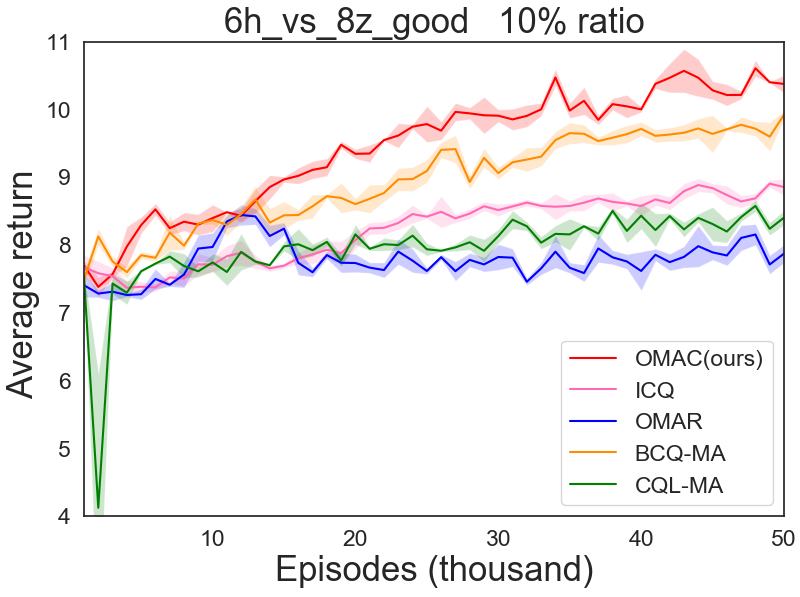}\hspace{0pt}
\includegraphics[width=2.3in]{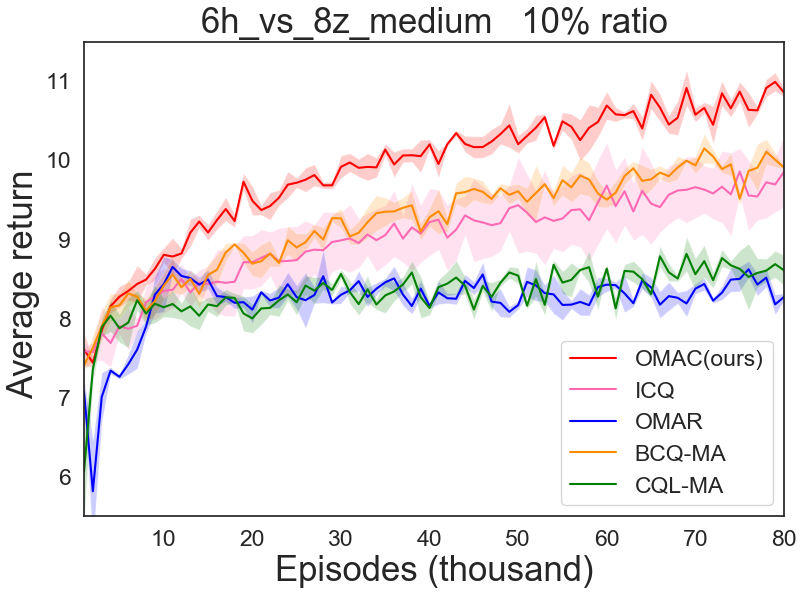}\hspace{0pt}
\includegraphics[width=2.3in]{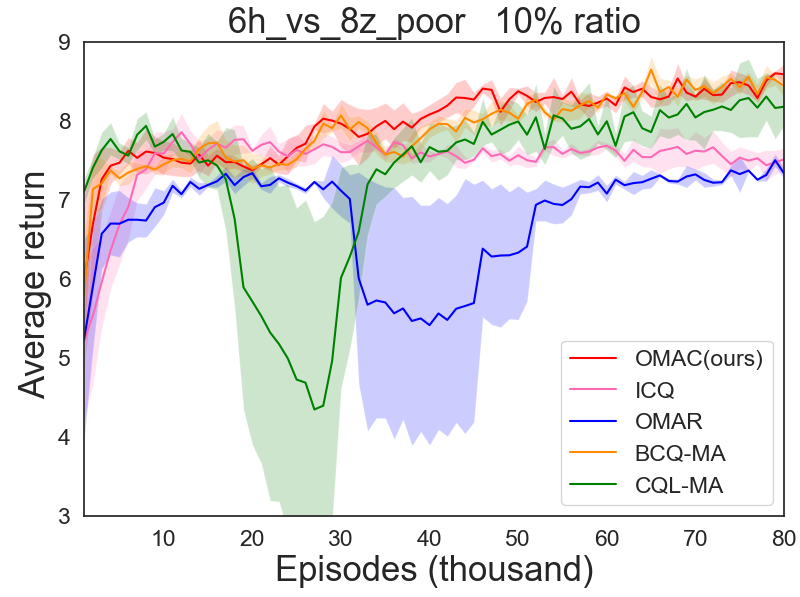}\hspace{0pt}
\end{minipage}%
\centering
\caption{Learning curves of data efficiency experiments on the offline 6h\_vs\_8z tasks.}

\label{fig:learning_curves_efficiency}
\end{figure*}

\subsection{Learning Curves of data efficiency experiments on offline SMAC tasks}
We further conduct experiments on SMAC map 6h\_vs\_8z with the size of the original datasets reduced to 50\% and 10\%. The learning curves of OMAC and baselines about data efficiency experiments on offline SMAC tasks are shown in Fig. \ref{fig:learning_curves_efficiency}.

\end{document}